\documentclass{article} 
\usepackage{iclr2027_conference,times}

\usepackage{amsmath,amsfonts,bm}

\def\eqref#1{equation~\ref{#1}}

\def\1{\bm{1}}

\DeclareMathAlphabet{\mathsfit}{\encodingdefault}{\sfdefault}{m}{sl}
\SetMathAlphabet{\mathsfit}{bold}{\encodingdefault}{\sfdefault}{bx}{n}

\usepackage{hyperref}
\usepackage{url}
\usepackage{amsmath,amsthm,amssymb,xcolor}
\newtheorem{proposition}{Proposition}
\newtheorem{corollary}{Corollary}

\usepackage{booktabs}
\usepackage{xcolor}
\usepackage{graphicx}
\usepackage{listings}
\usepackage{algorithm}
\usepackage{algpseudocode}
\lstdefinestyle{prompt}{
  basicstyle=\ttfamily\scriptsize, breaklines=true, breakatwhitespace=true,
  columns=fullflexible, keepspaces=true, frame=single,
  literate={⇒}{{$\Rightarrow$}}1 {∀}{{$\forall$}}1 {—}{{--}}1
}
\usepackage{listings}
\lstdefinestyle{lean}{
  basicstyle=\ttfamily\footnotesize, breaklines=true,
  columns=fullflexible, keepspaces=true,
  literate={ℝ}{{$\mathbb{R}$}}1 {∃}{{$\exists$}}1 {∈}{{$\in$}}1
           {→}{{$\to$}}1 {⟨}{{$\langle$}}1 {⟩}{{$\rangle$}}1
           {∀}{{$\forall$}}1 {≠}{{$\neq$}}1 {≤}{{$\le$}}1 {′}{{$'$}}1
}
\usepackage{amssymb}
\DeclareUnicodeCharacter{211D}{\ensuremath{\mathbb{R}}}   
\DeclareUnicodeCharacter{2203}{\ensuremath{\exists}}      
\DeclareUnicodeCharacter{2208}{\ensuremath{\in}}          
\DeclareUnicodeCharacter{2200}{\ensuremath{\forall}}      
\DeclareUnicodeCharacter{21D2}{\ensuremath{\Rightarrow}}  

\title{Can the Library Speak? Self-Advertised Method Selection for Formal Proving}

\author{
Xiaopeng Yuan$^{1}$, Suijin Wang$^{2}$, Yanli Wang$^{3}$, Haibo Jin$^{1}$, \\
\textbf{Peng Kuang$^{1}$, Jerry Wang$^{1}$, Lijun Yu$^{1}$, Haohan Wang$^{1}$} \\
$^{1}$University of Illinois at Urbana-Champaign \quad $^{2}$Tsinghua University \quad $^{3}$Imperial College London \\
\texttt{\{xyuan30,haohanw\}@illinois.edu}
}

\iclrfinalcopy 
\begin{document}

\maketitle
\lhead{Preprint}

\begin{abstract}
LLM-based formal provers can retrieve relevant lemmas and prior proofs, but relevance alone does not say whether a mathematical method can be used on the current theorem. A method has prerequisites, a target, an intended action, and obligations that its use leaves to prove. Methods that look equally related to a theorem may therefore differ substantially in whether they offer a plausible next step. We formulate this as an applicability-aware method-selection problem and introduce \emph{self-advertisement}: before candidates are ranked, a model generates a problem-specific proposal for each one, stating what part of the goal it targets, what action it would take, and what conditions that action requires. We organize 82 reusable methods from Putnam 2000--2014 as \emph{Method Contracts}, which pair applicability descriptions with Mathlib anchors, a checked example or scaffold, and expected proof obligations. A single batched call elicits proposals across the library; vague or unsupported proposals are demoted, yielding a ranked shortlist accompanied by inspectable claims about each candidate's use. We analyze when similarity-based representations cannot distinguish methods with different applicability, how errors in applicability estimates affect shortlist quality, and what a checked scaffold guarantees under its stated assumptions. Against lexical, embedding, and embedding-plus-LLM reranking baselines, self-advertisement achieves 95.0\% hit@5 on Putnam 2015–2025, compared with 84.2\% for the strongest reranker. On IMO ProofBench, it achieves 91.7\% compared with 88.3\%. These results indicate improved coverage of annotated methods in the retrieved shortlists, particularly on Putnam.
\end{abstract}

\section{Introduction}
Large language models can propose high-level approaches to mathematical problems~\citep{jiang2022draft, lightman2024let, shao2024deepseekmath}, and recent theorem-proving systems use them to generate Lean code, check it, and revise it in response to feedback~\citep{xin2025deepseek, ren2025deepseek, lin2026goedel, ospanov2026apollo, varambally2026hilbert}. Before a prover can use a mathematical method, however, it must decide whether that method offers a useful step for the current goal. A suggestion such as ``use induction'' leaves open what to induct on, which part of the goal the method would address, and what remains to be proved. These questions arise for every candidate method: what would it do in this problem, and what conditions would its use require? We study method selection as the task of eliciting and comparing such problem-specific proposals before passing a method to the prover.

Existing systems support formal proof construction at several levels. Premise selection~\citep{yang2023leandojo, mikula2024magnushammer} and Mathlib search~\citep{gao2024semantic, lu2026lean} provide relevant definitions and theorems; proof retrieval offers examples of how related problems were solved~\citep{wang2024lego, thompson2024rango}; and verifier-guided loops use Lean's errors to revise a failed attempt~\citep{ji2025leanabell, zheng2023lyra, thakur2023context}. These resources are valuable, but finding a related fact or proof does not by itself explain how a mathematical method fits the current goal. A retrieved inequality, for example, may require nonnegativity that has not been established, while a useful method may come from a problem that looks quite different on the surface. Lexical and embedding retrievers rank candidates by their similarity to the theorem~\citep{zhang2025qwen3, neelakantan2022text}, and an LLM reranker may reason about their applicability~\citep{sun2023chatgpt}. A relevance score alone, however, does not show why a method applies to the current goal.

\begin{figure}[t]
    \centering
    \includegraphics[width=\linewidth]{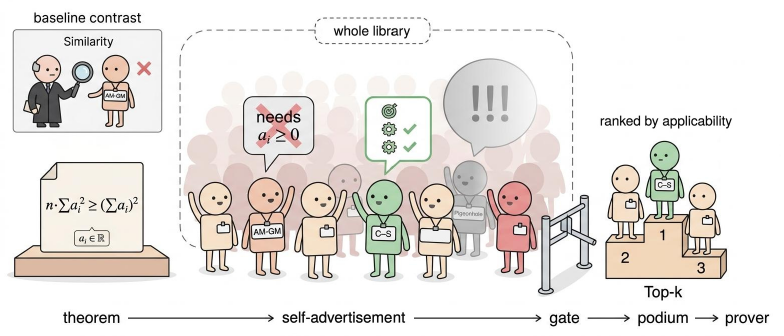}
    \caption{\textbf{Self-advertised method selection.} Given a theorem, a single model call asks every contract in the library how it would be used. A gate demotes proposals that are vague (Pigeonhole) or unsupported (AM-GM, since $a_i \ge 0$ is not given), and the top-ranked contract (Cauchy--Schwarz) goes to the prover. \emph{Inset:} similarity retrieval would pick AM-GM instead.}
    \label{fig:pipeline}
\end{figure}

The central problem is that surface similarity does not establish applicability. We therefore ask a more specific question than which methods look related: given a goal and a library of reusable methods, which methods could be applied, what part of the goal would each target, what action would it take there, and what conditions would that action require? A selector that commits to these answers makes its choices open to inspection rather than presenting only a relevance score. We evaluate whether such selections place suitable methods in the retrieved shortlist; whether the prover can then complete the remaining Lean obligations is a separate question.

We introduce \emph{self-advertisement}, an approach to method selection in which candidate methods propose how they could be used for the current goal before they are ranked. Instead of receiving only a relevance score, the prover receives a proposal that answers these questions for the selected method. We make this possible with \emph{Method Contracts}. Each contract describes a method's use and applicability conditions, together with execution guidance such as Mathlib anchors, a checked example or scaffold, and expected residual obligations. We construct 82 contracts from Putnam 2000--2014 materials and elicit their proposals in a single batched model call. Vague or unsupported proposals are demoted, so ranking depends on the concrete use a candidate offers rather than surface similarity alone. Figure~\ref{fig:pipeline} shows how contracts are proposed, screened, and ranked before the top-ranked contract is supplied to the Lean prover. Unlike an LLM reranker, which reorders candidates that similarity has already retrieved, self-advertisement considers every contract in the library and asks each to commit to a use that can be checked against the goal. We also show formally that a similarity-only selector can assign the same representation to goals that need different methods.

We evaluate method selection on Putnam 2015--2025 problems and IMO ProofBench problems against lexical, dense-embedding, and embedding-plus-LLM-reranking baselines, running every LLM-based selector with the same backbone under three backbones. Self-advertisement attains the highest recall@5 in all six backbone--benchmark settings, and on Putnam it also leads at hit@3, hit@5, and MRR under every backbone. With GPT-5.6 Luna, the backbone we use for proving, it reaches 95.0\% hit@5 on Putnam and 91.7\% on IMO ProofBench, compared with 84.2\% and 88.3\% for the strongest reranker. Rerankers remain more precise at rank~1 in most settings, so our claim concerns shortlist coverage rather than top-1 accuracy. A second-stage runoff narrows this gap: with GPT-5.6 Terra it raises Putnam hit@1 from 62.5\% to 72.5\%, above the strongest reranker at 67.5\%. In an Ax-Prover-style proof loop with a fixed compile budget, adding the top-ranked contract raises the proportion of proved problems from 5.8\% to 12.5\% on Putnam and from 10.0\% to 15.0\% on IMO ProofBench. This comparison adds the selected contract as a whole, so it does not isolate the contribution of selection, and most problems remain unsolved within the budget.

Our contributions are:
\begin{enumerate}
    \item \textbf{Method selection as an applicability problem.} We define method selection as distinct from similarity-based retrieval and show that a similarity-only selector can assign the same representation to goals needing different methods. We analyze how errors in applicability estimates affect the shortlist and what a checked scaffold guarantees under its assumptions.
    \item \textbf{Self-advertisement over Method Contracts.} Each contract links a method's intended action and required conditions to Mathlib anchors, a checked scaffold, and the obligations it leaves. In one batched call, each contract proposes a target, an action, and the conditions it requires for the current goal; vague or unsupported proposals are demoted before ranking.
    \item \textbf{A contract library and evaluation.} We build 82 contracts from Putnam 2000--2014 and evaluate on held-out Putnam 2015--2025 and IMO ProofBench. Self-advertisement reaches 95.0\% hit@5 on Putnam, compared with 84.2\% for the strongest reranker, and its coverage gains come largely from methods that do not look similar to the problem.
\end{enumerate}

\section{Related Work}
\paragraph{Knowledge for LLM-based formal proving.} LLM-based provers generate Lean proofs and revise them using verifier feedback, through whole-proof regeneration~\citep{xin2024deepseek, lin2026goedel}, step-level repair~\citep{liu2025proofaug, first2023baldur}, or agentic loops~\citep{achim2025aristotle, baba2025prover}. Premise selection~\citep{zhu2026premise, song2024lean} and Mathlib search~\citep{gao2026leansearch, asher2025leanexplore} supply formal facts; proof retrieval supplies worked examples~\citep{kasaura2026discovering}; and lemma libraries preserve results for reuse across problems~\citep{berlot2024library}. Informal sketches can also guide formal proof construction~\citep{wiedijk2003formal, lin2025lean}. These approaches provide different forms of useful knowledge, but retrieval typically returns a fact, lemma, or prior proof. Our retrieval unit is instead a mathematical method whose prerequisites, intended action, and execution guidance are represented together.

\paragraph{Selecting by applicability.} An LLM reranker can reason about whether a retrieved candidate fits a goal~\citep{qin2024large, weller2025rank1, zhuang2026rank}, although standard reranking exposes the result as a score or ordering rather than a structured account of how the candidate would be used. Earlier work on proof planning represents methods with preconditions and expected effects~\citep{bundy1988use, melis1999knowledge}, while adaptation-guided case retrieval favors cases that can be adapted to the new problem~\citep{leake1997case, smyth1998adaptation}. We build on these ideas; our contribution is to elicit problem-specific applicability proposals from a library of formalization methods and evaluate them for selection against lexical, embedding, and LLM-reranked retrieval.

 Self-advertisement combines reusable method knowledge with applicability-aware selection. Each Method Contract records a method's prerequisites, Mathlib guidance, and expected proof obligations. In one batched call, a model generates for each contract a proposal specifying its target in the current goal, the action it would take, and the conditions it requires; vague proposals are demoted before ranking. The resulting shortlist presents concrete, inspectable reasons for selecting each method and can guide a verifier-guided proof loop.

\section{Applicability-Aware Method Retrieval: A Formal Framework}
\label{sec:theory}

\subsection{Problem and contracts}
\label{sec:problem-contract}

Let $g$ be a Lean goal with local context $\Gamma$, and let $\mathcal C$ be a finite library of method contracts. A contract $c$ has parameters $x_c$, a selection-facing description, required conditions $\operatorname{Pre}_c(x_c)$, a target statement $G_c(x_c)$ with an intended action, an execution scaffold $S_c$, and a list of residual obligations $\operatorname{Obl}_c(x_c)$. The action reduces the target to the required conditions and the obligations: once $\operatorname{Pre}_c(\theta)$ and $\operatorname{Obl}_c(\theta)$ are proved for a binding $\theta$ of $x_c$, the action yields $G_c(\theta)$. A selector receives $g$, $\Gamma$, and the contract descriptions, and returns an ordered list $\pi(g)=(c_1,\ldots,c_{|\mathcal C|})$. Under self-advertisement, the selector first elicits a proposal for each contract, stating a binding $\hat\theta$, the target $G_c(\hat\theta)$, the action, and the conditions $\operatorname{Pre}_c(\hat\theta)$ that the step requires.

We distinguish semantic applicability from an observed label. A contract $c$ is \emph{applicable} to $g$ if some correct proof of $g$ contains a step that instantiates $c$ at a binding $\theta$: the step establishes $\operatorname{Pre}_c(\theta)$ from the context available at that point, proves $\operatorname{Obl}_c(\theta)$, obtains $G_c(\theta)$ by the action of $c$, and uses $G_c(\theta)$ in the rest of the proof. Let $\operatorname{App}(g)\subseteq\mathcal C$ denote the set of applicable contracts. In experiments we measure against a nonempty set $A(g)\subseteq\mathcal C$ of annotated reference contracts, which may contain several methods used in the reference solution. When the annotation is correct, $A(g)\subseteq\operatorname{App}(g)$, and the inclusion can be strict because other proofs may use other methods. A shortlist that contains a member of $A(g)$ therefore also contains an applicable contract, so agreement with $A(g)$ understates, rather than overstates, how often a shortlist contains a usable method.

\subsection{What a checked scaffold establishes}
\label{sec:scaffold}

\begin{proposition}[Conditional scaffold soundness]
\label{prop:scaffold}
Suppose a contract provides a Lean-checked term with no \texttt{sorry} or admitted proof holes, of type
\[
S_c:\;\forall x_c,\;\operatorname{Pre}_c(x_c)\to
\operatorname{Obl}_{c,1}(x_c)\to\cdots\to
\operatorname{Obl}_{c,r}(x_c)\to G_c(x_c).
\]
For any binding $\theta$ whose required conditions and obligations have checked proofs, applying $S_c$ yields a checked proof of $G_c(\theta)$.
\end{proposition}

\begin{proof}[Proof sketch]
Instantiate the checked term at $\theta$ and apply it to the checked proofs of its premises. Lean checks the resulting term.
\end{proof}

Proposition~\ref{prop:scaffold} applies only when the scaffold is checked in this parameterized form; a compiled example certifies only that example, not arbitrary new bindings. Even then, a self-advertised proposal identifies the remaining work for the new goal, $\operatorname{Pre}_c(\hat\theta)$ and $\operatorname{Obl}_c(\hat\theta)$; it does not discharge it.

\subsection{Similarity representations and the candidate-pool ceiling}
\label{sec:similarity}

Let $X=(g,\Gamma,\text{contract descriptions})$ denote the selector's input, and treat the annotated set $A=A(g)$ as random given $X$. Let $T=\phi(X)$ be the features available to a similarity-only selector, such as its vector of problem-to-contract similarity scores. The metrics hit@$k$ and recall@$k$ are defined in Appendix~\ref{app:metrics}. With information $F$, the best achievable expected hit@$k$ is
\[
V_k(F)=\mathbb E\!\left[\max_{S\subseteq\mathcal C,\;|S|=k}
\Pr(S\cap A\neq\varnothing\mid F)\right].
\]

\begin{proposition}[Information bound]
\label{prop:info}
If $T=\phi(X)$, then $V_k(X)\ge V_k(T)$. The inequality can be strict when distinct goals share the same similarity representation but require different shortlists, and the full descriptions distinguish them.
\end{proposition}

\begin{proof}[Proof sketch]
Because $T$ is a function of $X$, any shortlist chosen from $T$ can also be chosen from $X$. For a strict example, take $n>k$ equally likely goals with identical $T$ and distinct single reference contracts identifiable from $X$. Any $T$-based shortlist contains at most $k$ of these contracts and succeeds with probability at most $k/n$, whereas an $X$-based selector succeeds on every goal.
\end{proof}

Proposition~\ref{prop:info} is a statement about information, not a guarantee that a particular model extracts it. An LLM reranker reads full descriptions, but only of the contracts that a similarity stage passes to it, so similarity still determines what the reranker can consider.

\begin{corollary}[Candidate-pool ceiling]
\label{cor:ceiling}
Consider a two-stage selector that takes the top $m$ contracts $P_m(g)$ under a similarity-only selector and then orders them by any procedure, placing the remaining contracts after them. For every $k\le m$ and every goal $g$,
\[
\text{hit@}k\le\mathbf 1\bigl[P_m(g)\cap A(g)\neq\varnothing\bigr],
\qquad
\text{recall@}k\le\frac{|P_m(g)\cap A(g)|}{|A(g)|}.
\]
The two-stage selector's hit@$k$ and recall@$k$ are therefore bounded by the first-stage selector's hit@$m$ and recall@$m$, whichever model is used in the second stage.
\end{corollary}

Our reranking baselines follow this two-stage design, and every shortlist size $k$ we report satisfies $k\le m$, so an annotated contract that similarity ranks low cannot be recovered by any reranking backbone. Both kinds of baseline depend on similarity, either as the ranking signal or as the filter that decides what the reranker sees. Self-advertisement judges every contract from its proposal and does not rely on similarity in either role; its coverage is limited by the information in $X$ and by how accurately the model judges applicability, which we analyze next.

\subsection{Estimating applicability and ranking regret}
\label{sec:regret}

For a fixed description $X=x$, let
\[
w_c(x)=\mathbb E\!\left[\frac{\mathbf 1[c\in A]}{|A|}\;\middle|\;X=x\right].
\]
By linearity of expectation, the expected recall@$k$ of a shortlist $S$ is $R_x(S)=\sum_{c\in S}w_c(x)$, so ranking contracts by $w_c(x)$ maximizes expected recall@$k$. With a single reference contract $Y$, $w_c(x)=\Pr(Y=c\mid X=x)$ and $R_x(S)$ is also the expected hit@$k$. Self-advertisement can be viewed as estimating this ordering from each contract's proposal. Its scores need not be probabilities: a ranking depends only on the order of the scores, so the bound below requires only that some increasing transformation of them be close to $w_c(x)$.

\begin{proposition}[Top-$k$ ranking regret]
\label{prop:regret}
Let $S_k^*$ contain the $k$ contracts with the largest $w_c(x)$, let $\widehat S_k$ contain the $k$ contracts with the largest scores $\hat s_c$, and suppose there is a strictly increasing map $f$ such that $|f(\hat s_c)-w_c(x)|\le\epsilon$ for every $c\in S_k^*\cup\widehat S_k$. Then
\[
R_x(S_k^*)-R_x(\widehat S_k)\le 2k\epsilon.
\]
\end{proposition}

\begin{proof}[Proof sketch]
Let $\hat w_c=f(\hat s_c)$. Since $f$ is strictly increasing, $\widehat S_k$ also contains the $k$ contracts with the largest $\hat w_c$. Write $R_x(S_k^*)-R_x(\widehat S_k)=\bigl[\sum_{c\in S_k^*}\hat w_c-\sum_{c\in\widehat S_k}\hat w_c\bigr]+\sum_{c\in S_k^*}(w_c-\hat w_c)-\sum_{c\in\widehat S_k}(w_c-\hat w_c)$. The bracket is nonpositive because $\widehat S_k$ maximizes the transformed scores, and each of the remaining $2k$ terms has magnitude at most $\epsilon$.
\end{proof}

The bound is conditional: it does not assert that our scores meet its premise, and it is a worst-case statement rather than a prediction of the measured gains. Its useful consequence is that only estimation errors on contracts that are, or appear to be, among the top $k$ affect the shortlist. Demoting vague or unsupported proposals acts on exactly these errors: it moves contracts that appear to be among the top $k$ but cannot state a checkable use below every supported nomination. A second source of error is that self-advertisement assigns each proposal an absolute score rather than comparing proposals directly, so scores from different contracts need not lie on the common scale that the premise requires. A runoff among the leading proposals (Section~\ref{sec:runoff}) addresses this where it matters: it compares them directly and can correct their relative order without rescoring the whole library.

\begin{proposition}[Multi-method coverage]
\label{prop:coverage}
Let $A\subseteq\mathcal C$ be the random set of annotated contracts for a goal with description $X=x$, and define the expected hit@$k$ objective
\[
H_x(S)=\Pr(S\cap A\neq\varnothing\mid X=x).
\]
Then $H_x$ is monotone and submodular in the candidate set $S$. A greedy selector using the true marginal gains, adding the contract that maximizes $H_x(S\cup\{c\})-H_x(S)$ at each step, attains at least $1-(1-1/k)^k$ times the optimal size-$k$ coverage, and hence at least $1-1/e$ times the optimum.
\end{proposition}

\begin{proof}[Proof sketch]
For any fixed annotated set $A$, adding a contract can turn a miss into a hit only if the current shortlist has not hit $A$; enlarging the shortlist cannot increase that marginal gain. Expectation preserves this diminishing-returns property. The greedy guarantee then follows from the classical result for monotone submodular maximization under a cardinality constraint~\citep{nemhauser1978analysis}.
\end{proof}

The two metrics therefore call for different selection rules. Expected recall@$k$ is additive over contracts, so sorting by $w_c$ is optimal for it. Expected hit@$k$ is submodular: once a shortlist contains a method that is likely to apply, a near-duplicate adds little, so sorting per-contract scores need not maximize it. Our implementation sorts per-contract scores, which is optimal for expected recall@$k$, and also for expected hit@$k$ whenever a goal has a single reference contract, since the two objectives then coincide. When several methods apply, Proposition~\ref{prop:coverage} identifies the objective that a diversity-aware selector would target.

\subsection{Selection and downstream proof success}
\label{sec:downstream}

Let $E_k$ denote the event that the first $k$ positions contain an annotated contract, and $Z_B$ the event that the prover completes a Lean proof within a budget of $B$ compile calls. For any prover,
\[
\Pr(Z_B)=\Pr(Z_B\mid\neg E_k)+\Pr(E_k)\,\Delta_k,
\qquad
\Delta_k=\Pr(Z_B\mid E_k)-\Pr(Z_B\mid\neg E_k).
\]
Selection controls $\Pr(E_k)$, which the shortlist metrics measure. The gap $\Delta_k$ measures how much a well-chosen contract helps the prover, and depends on the contract's execution guidance and the difficulty of its residual obligations. For a fixed $\Delta_k>0$, raising $\Pr(E_k)$ raises proof success, so better selection translates into more proofs whenever the selected contracts are useful to the prover.

Which selection metric matters depends on how the ranking is used. A prover that tries contracts in order within its budget needs a hit among the positions it can reach, which hit@$k$ measures. Many problems also require several methods, so a proof along the reference route needs every method in $A(g)$, not just one of them; recall@$k$ measures how much of that route a shortlist supplies. A second stage, such as the runoff, also chooses its top contract from this ranking, so the coverage of the list it re-ranks bounds the rank-1 accuracy it can reach.

\section{Method}

\label{sec:method}

Our method centers on self-advertisement, which elicits a proposal from every contract in a single call and ranks the contracts by whether they commit to a checkable use (Section~\ref{sec:selfad}). Self-advertisement is designed for shortlist coverage; when the first choice matters most, an optional runoff compares the leading proposals directly to improve rank-1 accuracy (Section~\ref{sec:runoff}). Both read only the selection side of each contract; the execution side of the selected contract is then passed to the prover. Section~\ref{sec:card-construction} describes how the 82 contracts are built.

\subsection{Self-advertisement}
\label{sec:selfad}

Self-advertisement replaces the question of how closely a contract resembles the goal with the question of how the contract would be used. In a single model call, the model reads the goal, its local context, and the selection-facing description of every contract, and returns a proposal for each contract. The proposal of contract $c$ makes a commitment that can be checked against the goal: a target $t_c$, which instantiates $G_c(\hat\theta)$ in terms of the goal's own objects; an action $a_c$; and the required facts $\widehat{\operatorname{Pre}}_c$, which instantiate $\operatorname{Pre}_c(\hat\theta)$. It also cites verified material $\sigma_c$ from the contract that supports this commitment, and records whether the contract offers itself ($o_c$), its role $\rho_c\in\{\text{main},\text{support}\}$, and an applicability score $\hat s_c$. 

A commitment that is vague or unsupported cannot be checked, so a main nomination must pass the commitment gate
\[
\operatorname{Gate}(c)=o_c\wedge\operatorname{Spec}(t_c)\wedge\operatorname{Spec}(a_c)\wedge(\sigma_c\neq\varnothing),
\]
where $\operatorname{Spec}$ holds when a target or action names a concrete object of the goal rather than a generic phrase. A main nomination that fails the gate is demoted to a supporting role with its score unchanged. The tiered ranking then orders contracts lexicographically by $(\tau_c,-\hat s_c)$, where $\tau_c=1$ for main nominations that pass the gate, $\tau_c=2$ for other contracts that offer themselves, and $\tau_c=3$ otherwise. The gate thus moves down the proposals that would otherwise occupy leading positions without a checkable use, which are the errors that matter most for the shortlist (Proposition~\ref{prop:regret}). 

\subsection{Runoff}
\label{sec:runoff}

Self-advertisement targets shortlist coverage, but a prover that receives a single contract depends on hit@1 (Section~\ref{sec:downstream}). Its absolute integer scores often leave the leading proposals close or tied, with ties broken by contract identifier rather than content. Since only errors among the leading candidates affect the shortlist (Proposition~\ref{prop:regret}), the runoff re-ranks the top $q$ contracts by comparing their proposals. It uses the same backbone, so any gain comes from direct comparison rather than a stronger model.

In one pass of $q-1$ calls from the top of the window down, the model sees each contract, the contract above it, and the proposal that contract wrote earlier in the pass; it writes a proposal for the lower contract and judges whether it fits the proof better, and if so the two swap. Contracts outside the window keep their order.

\subsection{Contract library}
\label{sec:card-construction}

Our library contains 82 Method Contracts built from solutions to Putnam 2000--2014 problems; no problem from the evaluation sets, Putnam 2015--2025 and IMO ProofBench, is used as a source. Each contract has two sides, following the definition in Section~\ref{sec:problem-contract}. The \emph{selection side} is what a selector reads: the method in one sentence, the kind of target it typically addresses, the conditions under which it applies, cases in which it should not be used, and a manifest naming the verified material on its execution side, such as its scaffold and worked instance, without their Lean code. The \emph{execution side} is supplied only to the prover: a Lean scaffold whose parameters make the method's bindings explicit, a compiled worked instance, the Mathlib declarations it relies on, its required conditions stated as Lean hypotheses, and the residual obligations left after its step.

We build contracts in three steps, using GPT-5.6 Sol for both extraction and drafting. First, the model extracts the methods used in the reference solutions and proposes merges between methods that play the same role in different problems; the authors review every extracted method and every proposed merge. Second, the model drafts both sides of each contract from its source solutions. Third, every Lean artifact on the execution side is compiled with Lean v4.27.0 and Mathlib at commit \texttt{a3a10db} (release v4.27.0). All 82 contracts pass this check: each provides a scaffold stated for arbitrary parameters and a compiled worked instance.

\section{Experiment}
\subsection{Setup}
\label{sec:setup}

\paragraph{Baselines and base models.} We compare against random ordering, BM25~\citep{robertson2009probabilistic}, TF-IDF cosine~\citep{salton1988term}, and three embedding models (OpenAI text-embedding-3-small, text-embedding-3-large~\citep{neelakantan2022text}, and Qwen3-Embedding-0.6B~\citep{zhang2025qwen3}). For each embedding model, we also evaluate a listwise LLM reranker~\citep{sun2023chatgpt} over its top-10 candidates. Self-advertisement generates goal-specific proposals for all 82 contracts in one batched call, followed by the commitment gate and tiered ranking (Section~\ref{sec:selfad}); the runoff re-ranks the top 10 of this ranking (Section~\ref{sec:runoff}), so both second stages operate on the same number of candidates. We run all LLM-based selectors with DeepSeek V4.1 Flash~\citep{xu2026deepseek}, GPT-5.6 Luna, and GPT-5.6 Terra~\citep{singh2025openai}; within each block of Table~\ref{tab:selection}, every LLM-based selector uses the same backbone.

\paragraph{Benchmarks and labels.} We evaluate on Putnam 2015--2025 (120 problems) and IMO ProofBench (60 problems). The contract library is built from Putnam 2000--2014 solutions, so no evaluation problem is a contract source. Each problem is labeled with the contracts whose method appears in its official solution: GPT-5.6 Sol proposed labels from the official solution, and a human checked every label. Labels were fixed before any selector was run, so hit@$k$ measures agreement with these labels, a conservative estimate of applicability (Section~\ref{sec:problem-contract}).

\begin{table}[t]
\centering
\small
\setlength{\tabcolsep}{1.5pt}
\begin{tabular*}{\textwidth}{@{}l@{\extracolsep{\fill}}cccccc cccccc@{}}
\toprule
& \multicolumn{6}{c}{Putnam 2015--2025 ($n{=}120$)} & \multicolumn{6}{c}{IMO ProofBench ($n{=}60$)} \\
\cmidrule(lr){2-7} \cmidrule(lr){8-13}
Selector & hit@1 & hit@3 & hit@5 & R@3 & R@5 & MRR & hit@1 & hit@3 & hit@5 & R@3 & R@5 & MRR \\
\midrule
\multicolumn{13}{l}{\emph{No LLM}} \\
Random             & .058 & .158 & .225 & .038 & .053 & .141 & .083 & .167 & .267 & .041 & .076 & .158 \\
BM25               & .208 & .400 & .508 & .104 & .140 & .337 & .183 & .333 & .400 & .103 & .147 & .295 \\
TF-IDF cosine      & .267 & .442 & .525 & .118 & .144 & .382 & .233 & .367 & .433 & .135 & .165 & .318 \\
OpenAI emb-3-small & .325 & .533 & .683 & .157 & .229 & .463 & .267 & .517 & .617 & .213 & .263 & .411 \\
OpenAI emb-3-large & .425 & .667 & .725 & .204 & .267 & .556 & .367 & .583 & .717 & .268 & .339 & .523 \\
Qwen3-Emb-0.6B     & .450 & .675 & .775 & .217 & .291 & .579 & .350 & .617 & .767 & .282 & .380 & .510 \\
\midrule
\multicolumn{13}{l}{\emph{GPT-5.6 Terra}} \\
OpenAI-small + rerank & .642 & .767 & .775 & .285 & .324 & .704 & .483 & .717 & .767 & .336 & .397 & .607 \\
OpenAI-large + rerank & .667 & .800 & .842 & .315 & .373 & .742 & .600 & .783 & .900 & .378 & .455 & .712 \\
Qwen3 + rerank        & .675 & .808 & .858 & .338 & .407 & .751 & \textbf{.633} & .783 & .900 & .373 & .469 & .724 \\
Self-adv.\ (ours)     & .625 & .908 & .975 & .342 & .472 & .778 & .517 & \textbf{.867} & .883 & \textbf{.422} & .494 & .694 \\
\quad + runoff & \textbf{.725} & \textbf{.950} & \textbf{.983} & \textbf{.374} & \textbf{.503} & \textbf{.831} & .600 & .833 & \textbf{.933} & .409 & \textbf{.512} & \textbf{.736} \\
\midrule
\multicolumn{13}{l}{\emph{GPT-5.6 Luna}} \\
OpenAI-small + rerank & .592 & .742 & .783 & .286 & .321 & .673 & .517 & .750 & .783 & .350 & .400 & .628 \\
OpenAI-large + rerank & \textbf{.692} & .792 & .825 & .312 & .361 & .747 & \textbf{.583} & \textbf{.833} & .883 & .404 & .460 & \textbf{.714} \\
Qwen3 + rerank        & .683 & .817 & .842 & .337 & .385 & .756 & .567 & .817 & .883 & .407 & .468 & .688 \\
Self-adv.\ (ours)     & .658 & \textbf{.917} & .950 & .346 & .446 & .787 & .550 & .817 & .917 & .414 & \textbf{.554} & .693 \\
\quad + runoff & .683 & .908 & \textbf{.958} & \textbf{.352} & \textbf{.461} & \textbf{.800} & .567 & .817 & \textbf{.933} & \textbf{.417} & .538 & .703 \\
\midrule
\multicolumn{13}{l}{\emph{DeepSeek V4.1 Flash}} \\
OpenAI-small + rerank & .550 & .750 & .775 & .262 & .305 & .643 & .550 & .767 & .767 & .354 & .388 & .640 \\
OpenAI-large + rerank & .567 & .775 & .825 & .293 & .367 & .673 & .550 & \textbf{.833} & .850 & .383 & .446 & .685 \\
Qwen3 + rerank        & .575 & .783 & .842 & .307 & .371 & .687 & \textbf{.617} & .783 & .850 & .388 & .449 & \textbf{.705} \\
Self-adv.\ (ours)     & .600 & .808 & \textbf{.900} & .293 & .393 & .719 & .517 & .817 & .883 & .395 & .493 & .672 \\
\quad + runoff & \textbf{.608} & \textbf{.842} & .892 & \textbf{.308} & \textbf{.399} & \textbf{.735} & .517 & \textbf{.833} & \textbf{.917} & \textbf{.412} & \textbf{.523} & .684 \\
\bottomrule
\end{tabular*}
\caption{Shortlist quality against annotated reference contracts. hit@$k$: fraction of problems whose top-$k$ list contains at least one annotated contract. R@$k$: fraction of annotated contracts that appear in the top $k$. Bold: best within each LLM block; ties are all bolded.}
\label{tab:selection}
\end{table}

\subsection{Main Results}
\label{sec:results-selection}

Lexical matching reaches at most .525 hit@5 on Putnam, dense embeddings reach .775, and LLM reranking of the embedding shortlist adds another 7--8 points (Table~\ref{tab:selection}). Self-advertisement attains the highest recall@5 in all six backbone--benchmark settings, and on Putnam it also leads at hit@3, hit@5, and MRR under every backbone. Its lead grows with the backbone while the rerankers plateau: from DeepSeek V4.1 Flash to GPT-5.6 Luna to GPT-5.6 Terra, its Putnam hit@5 rises from .900 to .950 to .975, whereas the best reranker stays at .842, .842, and .858, as Corollary~\ref{cor:ceiling} predicts for a reranker that cannot recover contracts its embedding stage missed. The lead under DeepSeek also shows that it does not depend on the model family used to draft contracts and propose labels. On IMO ProofBench, whose problems come from different competitions than the Putnam-derived library, self-advertisement still leads at recall@5 under all three backbones (.494, .554, and .493, against at most .469, .468, and .449), and with the runoff it gives the best hit@5 under every backbone (.933, .933, and .917). Since one problem corresponds to 1.7 points there, we read these results as broader coverage at comparable top-5 hit rates.

Coverage and rank-1 accuracy diverge: rerankers remain more precise at rank~1 in most settings, and the runoff targets this gap by comparing the leading proposals. With GPT-5.6 Terra, it raises Putnam hit@1 from .625 to .725, above the strongest reranker (.675), and IMO ProofBench hit@1 from .517 to .600, and MRR becomes best in both Terra blocks (.831 and .736). With DeepSeek V4.1 Flash, it gives the best Putnam hit@1 and MRR in its block (.608 and .735). With GPT-5.6 Luna the changes are smaller, at most three problems per hit metric, and coverage is largely preserved.

\subsection{Applicability Beyond Similarity}
\begin{figure}[t]
    \centering
    \includegraphics[width=\textwidth, trim={0 0.45in 0 0}, clip]{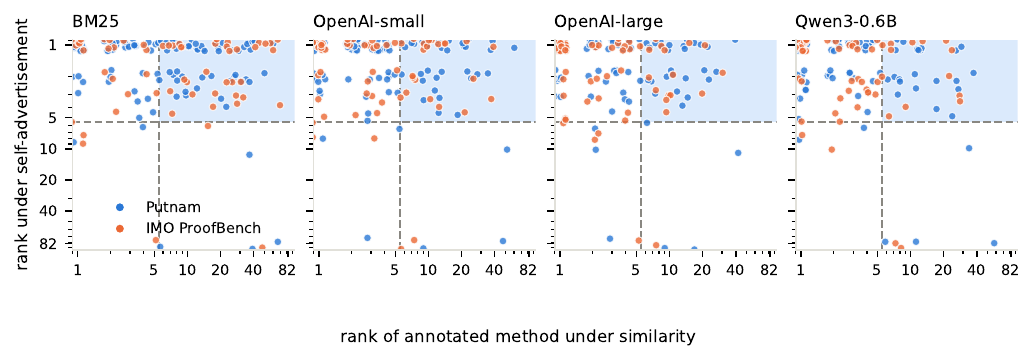}\par
    {\sffamily\footnotesize rank of annotated method under similarity}
    \caption{Each dot is one problem; $y$ is the rank of its highest-ranked annotated method under self-advertisement (GPT-5.6 Luna, shared across panels), and $x$ is that method's rank by problem--method similarity under the panel's retriever. Dashed lines mark the top five; the shaded region holds methods that similarity ranks outside its top five but self-advertisement places inside.}
    \label{fig:rank_vs_rank}
\end{figure}

Figure~\ref{fig:rank_vs_rank} shows where these coverage gains come from. If self-advertisement merely rearranged methods that already looked similar to the problem, annotated methods near the top of its ranking would also sit near the top of the similarity rankings. Instead, across BM25 and all three embedding retrievers, the shaded regions contain annotated methods that similarity ranks outside its top five but self-advertisement places inside its top five. Many of them are so dissimilar to the problem that they never enter the rerankers' candidate pool, so no reranking backbone can recover them. These are the cases that motivate our approach: a method can offer an applicable step toward the goal even when its description is not especially similar to the problem statement. Because self-advertisement judges every contract rather than only the candidates that similarity retrieves, it brings such methods into the shortlist, and each arrives with a target, an action, and required conditions that can be checked against the goal. The figure thus connects the coverage gains in Table~\ref{tab:selection} to methods that similarity-based candidate selection tends to overlook.

\subsection{Downstream Proving}

\begin{table}[t]
\centering
\small
\setlength{\tabcolsep}{4pt}
\renewcommand{\arraystretch}{0.95}
\begin{tabular}{@{}lcccc@{\hspace{10pt}}ccc@{}}
\toprule
 & \multicolumn{4}{c}{Putnam 2015--2025} & \multicolumn{3}{c}{IMO ProofBench} \\
\cmidrule(lr){2-5} \cmidrule(lr){6-8}
Method & ALG & ANA & DISC & All & Basic & Adv. & All \\
\midrule
Base & 0.0 & 0.0 & 0.0 & 0.0 & 3.3 & 0.0 & 1.7 \\
pass@20 & 2.3 & 0.0 & 0.0 & 0.8 & 3.3 & 0.0 & 1.7 \\
APOLLO & 4.5 & 0.0 & 2.2 & 2.5 & 13.3 & 0.0 & 6.7 \\
AxProverBase & \textbf{9.1} & 3.3 & 4.3 & 5.8 & 20.0 & 0.0 & 10.0 \\
\midrule
+ Self-adv.\ contract (ours) & \textbf{9.1} & \textbf{10.0} & \textbf{17.4} & \textbf{12.5} & \textbf{26.7} & \textbf{3.3} & \textbf{15.0} \\
\bottomrule
\end{tabular}
\caption{Downstream proving success rate (\%) with GPT-5.6 Luna under a budget of 10 Lean compilations per problem.}
\label{tab:proving}
\end{table}

As a first check of whether selected contracts help a prover, Table~\ref{tab:proving} reports proof success with GPT-5.6 Luna and a budget of ten Lean compilations per problem. Adding the top-ranked contract to the Ax-Prover-style loop more than doubles its success rate on Putnam, from 5.8\% to 12.5\%, and raises it from 10.0\% to 15.0\% on IMO ProofBench. On Putnam the gains are concentrated in analysis (3.3\% to 10.0\%) and discrete problems (4.3\% to 17.4\%), while algebra is unchanged at 9.1\%; on IMO ProofBench the contract adds Basic problems (20.0\% to 26.7\%) and the only Advanced problem solved by any arm. Base and pass@20 solve almost no problems, even though pass@20 uses twice the budget, and APOLLO trails the Ax-Prover-style control on both benchmarks. The comparison supplies the selected contract as a whole, so the gain reflects selection and contract content together (Section~\ref{sec:downstream}); most problems remain unsolved within this budget.

\section{Conclusion}
We presented self-advertisement, a method-selection approach that asks each Method Contract to propose a target, action, and required conditions for the current theorem. Across Putnam and IMO ProofBench, it improves coverage of annotated methods in shortlists, including methods ranked low by lexical or embedding similarity. Adding the top-ranked contract to a fixed-budget proof loop also increases the number of problems proved on both benchmarks. These results support selecting proof methods by their proposed use, while leaving two limits clear: shortlist coverage does not guarantee a completed proof, and the downstream comparison adds the contract as a whole rather than isolating selection from its execution guidance. Future work should separate these contributions and help provers carry promising method proposals through to verified proofs.

\section*{AI Use Statement}
In this work, we used generative AI tools for the following tasks. GPT-5.6 Sol extracted methods from reference solutions, proposed merges between them, and drafted both sides of each Method Contract; the authors reviewed every extracted method and proposed merge, and every Lean artifact was compiled and checked. GPT-5.6 Sol also proposed the reference labels for evaluation problems, and a human checked every label before any selector was run. Large language models are components of the evaluated systems (Section~\ref{sec:setup}). We also used generative AI to refine the statements of theoretical results and their proof sketches, to suggest the structure of parts of the paper, and to edit the text for readability.

\section*{Ethics Statement}
This work uses publicly available competition problems and official solutions (Putnam and IMO ProofBench) and involves no human subjects or personal data. The method selects proof strategies for formal theorem proving and raises no foreseeable risks beyond those of the underlying language models.

\section*{Reproducibility Statement}
Section~\ref{sec:card-construction} describes how the 82 contracts were built, including the Lean version (v4.27.0) and Mathlib commit (\texttt{a3a10db}) used to check every execution-side artifact; Appendix~\ref{app:env} lists the full pinned environment, and Figure~\ref{fig:contract-example} shows an example contract. Section~\ref{sec:setup} describes the baselines, backbones, benchmarks, and labeling procedure. The tiered ranking breaks ties by contract identifier, so it is deterministic given the model outputs.

\bibliography{iclr2027_conference}
\bibliographystyle{iclr2027_conference}

\appendix
\section{Selection metrics}
\label{app:metrics}

For a goal $g$, a selector returns an ordering $\pi(g)=(c_1,\ldots,c_{|\mathcal C|})$ of the whole library, and $A(g)\subseteq\mathcal C$ is the nonempty set of annotated contracts. Two-stage selectors place the contracts outside their candidate pool after the pool, so every selector orders the full library and every annotated contract has a rank. For a cutoff $k$, write $\pi_k(g)=\{c_1,\ldots,c_k\}$, and let $\operatorname{rank}(g)=\min\{i: c_i\in A(g)\}$ be the position of the first annotated contract. Over the $N$ problems of a benchmark we report
\begin{align*}
\text{hit@}k &= \frac{1}{N}\sum_{g}\mathbf 1\bigl[\pi_k(g)\cap A(g)\neq\varnothing\bigr],\\
\text{recall@}k &= \frac{1}{N}\sum_{g}\frac{|\pi_k(g)\cap A(g)|}{|A(g)|},\\
\text{MRR} &= \frac{1}{N}\sum_{g}\frac{1}{\operatorname{rank}(g)}.
\end{align*}
Table~\ref{tab:selection} abbreviates recall@$k$ as R@$k$. hit@$k$ is the fraction of problems whose shortlist contains at least one annotated method; recall@$k$ measures how much of the annotated set the shortlist covers; MRR rewards placing the first annotated contract early. For every problem, $\text{recall@}k\le\text{hit@}k$, both are nondecreasing in $k$, and the two coincide when $|A(g)|=1$.

\subsection{Multi-method coverage}
\label{app:coverage}

Section~\ref{sec:regret} notes that expected recall@$k$ and expected hit@$k$ call for different selection rules. This section proves the result that makes the second objective precise.

\noindent\textbf{Proposition~\ref{prop:coverage} (restated).}
\emph{Let $A\subseteq\mathcal C$ be the random set of annotated contracts for a goal with description $X=x$, and define the expected hit@$k$ objective $H_x(S)=\Pr(S\cap A\neq\varnothing\mid X=x)$. Then $H_x$ is monotone and submodular in the candidate set $S$. A greedy selector using the true marginal gains, adding at each step the contract that maximizes $H_x(S\cup\{c\})-H_x(S)$, attains at least $1-(1-1/k)^k$ times the optimal size-$k$ coverage, and hence at least $1-1/e$ times the optimum.}

\begin{proof}
Fix a realization of $A$ and let $h_A(S)=\mathbf 1[S\cap A\neq\varnothing]$. If $S\subseteq S'$ then $S\cap A\subseteq S'\cap A$, so $h_A(S)\le h_A(S')$: $h_A$ is monotone. For $c\notin S'$, the marginal gain of $c$ at $S$ is
\[
h_A(S\cup\{c\})-h_A(S)=\mathbf 1[c\in A]\cdot\mathbf 1[S\cap A=\varnothing],
\]
which is nonincreasing in $S$ because $\mathbf 1[S'\cap A=\varnothing]\le\mathbf 1[S\cap A=\varnothing]$ when $S\subseteq S'$: $h_A$ is submodular. Since $H_x(S)=\mathbb E[h_A(S)\mid X=x]$ is an average of the functions $h_A$ with nonnegative weights, and both properties are preserved under such averages, $H_x$ is monotone and submodular, with $H_x(\varnothing)=0$.

The greedy guarantee is the classical bound of \citet{nemhauser1978analysis}; we include the short argument. Let $S^{\mathrm{opt}}$ maximize $H_x$ over sets of size $k$, and let $S_i$ be the greedy set after $i$ steps, with $S_0=\varnothing$. For each $i<k$, monotonicity and submodularity give
\[
H_x(S^{\mathrm{opt}})
\le H_x(S_i\cup S^{\mathrm{opt}})
\le H_x(S_i)+\!\!\sum_{c\in S^{\mathrm{opt}}\setminus S_i}\!\!\bigl[H_x(S_i\cup\{c\})-H_x(S_i)\bigr]
\le H_x(S_i)+k\bigl[H_x(S_{i+1})-H_x(S_i)\bigr],
\]
where the last step uses $|S^{\mathrm{opt}}\setminus S_i|\le k$ and the fact that the greedy step picks the largest marginal gain. Writing $\delta_i=H_x(S^{\mathrm{opt}})-H_x(S_i)$, this rearranges to $\delta_{i+1}\le(1-1/k)\,\delta_i$, so $\delta_k\le(1-1/k)^k\,\delta_0=(1-1/k)^k\,H_x(S^{\mathrm{opt}})$. Hence $H_x(S_k)\ge\bigl(1-(1-1/k)^k\bigr)H_x(S^{\mathrm{opt}})\ge(1-1/e)\,H_x(S^{\mathrm{opt}})$, using $(1-1/k)^k\le e^{-1}$.
\end{proof}

The two metrics therefore call for different selection rules. Expected recall@$k$ is additive over contracts, $R_x(S)=\sum_{c\in S}w_c(x)$, so sorting by $w_c(x)$ is optimal for it, and Proposition~\ref{prop:regret} bounds the loss from sorting by estimates instead. Expected hit@$k$ is submodular: once a shortlist contains a method that is likely to apply, a near-duplicate adds little, so sorting per-contract scores need not maximize it. Our implementation sorts per-contract scores, which is optimal for expected recall@$k$, and also for expected hit@$k$ whenever a goal has a single reference contract, since the two objectives then coincide (Appendix~\ref{app:metrics}). When several methods apply, Proposition~\ref{prop:coverage} identifies the objective that a diversity-aware selector would target; we do not run the greedy procedure.

\section{Proofs}
\label{app:proofs}

This appendix gives full proofs of Propositions~\ref{prop:scaffold}--\ref{prop:regret} and Corollary~\ref{cor:ceiling}; Appendix~\ref{app:coverage} proves Proposition~\ref{prop:coverage}.

\subsection{Conditional scaffold soundness}

\noindent\textbf{Proposition~\ref{prop:scaffold} (restated).}
\emph{Suppose a contract provides a Lean-checked term with no \texttt{sorry} or admitted proof holes, of type
\[
S_c:\ \forall x_c,\ \operatorname{Pre}_c(x_c)\to\operatorname{Obl}_{c,1}(x_c)\to\cdots\to\operatorname{Obl}_{c,r}(x_c)\to G_c(x_c).
\]
For any binding $\theta$ whose precondition and obligations have checked proofs, applying $S_c$ yields a checked proof of $G_c(\theta)$.}

\begin{proof}
Lean's kernel accepts a term only together with its type, and it accepts an application $f\,a$ whenever $f$ has a dependent function type $\Pi\,y{:}\alpha,\ \beta(y)$ and $a$ has type $\alpha$, assigning the application the type $\beta(a)$. Instantiating the universal quantifier, $S_c\,\theta$ is accepted with type $\operatorname{Pre}_c(\theta)\to\operatorname{Obl}_{c,1}(\theta)\to\cdots\to\operatorname{Obl}_{c,r}(\theta)\to G_c(\theta)$. Let $p$ be a checked term of type $\operatorname{Pre}_c(\theta)$ and $q_1,\ldots,q_r$ checked terms of types $\operatorname{Obl}_{c,1}(\theta),\ldots,\operatorname{Obl}_{c,r}(\theta)$. Applying the rule $r+1$ times, the term $S_c\,\theta\,p\,q_1\cdots q_r$ is accepted with type $G_c(\theta)$. The premise that $S_c$ contains no \texttt{sorry} matters because \texttt{sorry} is an axiom that inhabits every type, so a term built from it would type-check without proving anything. Since $S_c$, $p$, and the $q_i$ are checked without that axiom, their application depends only on the axioms they depend on, and the result is a proof of $G_c(\theta)$ in the environment in which they were checked (Appendix~\ref{app:env}).
\end{proof}

The proposition concerns reuse, not discovery. It does not say that $\theta$ is a useful binding for the goal, that $\operatorname{Pre}_c(\theta)$ or the obligations are provable, or that $G_c(\theta)$ advances the proof; it says only that once these are discharged, the method's own step needs no further checking.

\subsection{Information bound}

\noindent\textbf{Proposition~\ref{prop:info} (restated).}
\emph{If $T=\phi(X)$, then $V_k(X)\ge V_k(T)$. The inequality can be strict when distinct goals share the same similarity representation but require different shortlists, and the full descriptions distinguish them.}

\begin{proof}
For each value $t$ of $T$, choose a set $S(t)$ with $|S(t)|=k$ that attains the maximum in the definition of $V_k(T)$; a maximizer exists because $\mathcal C$ is finite. Then
\[
V_k(T)=\mathbb E\bigl[\Pr(S(T)\cap A\neq\varnothing\mid T)\bigr]
=\Pr\bigl(S(T)\cap A\neq\varnothing\bigr)
=\mathbb E\bigl[\Pr(S(\phi(X))\cap A\neq\varnothing\mid X)\bigr],
\]
where the middle equality is the tower property and the last one uses that $S(\phi(X))$ is a function of $X$. For every $x$, the set $S(\phi(x))$ is one of the sets over which the inner maximum in $V_k(X)$ is taken, so $\Pr(S(\phi(x))\cap A\neq\varnothing\mid X=x)\le\max_{|S|=k}\Pr(S\cap A\neq\varnothing\mid X=x)$. Taking expectations over $X$ gives $V_k(T)\le V_k(X)$.

For strictness, fix $n$ with $k<n\le|\mathcal C|$, and let $X$ be uniform over $n$ distinct values $x_1,\ldots,x_n$ with $\phi(x_i)=t$ for every $i$. Let $c_1,\ldots,c_n$ be distinct contracts and set $A=\{c_i\}$ whenever $X=x_i$. Given $X=x_i$, any shortlist containing $c_i$ hits with probability one, so $V_k(X)=1$. Given $T=t$, which carries no information about $i$, a shortlist $S$ hits with probability $\tfrac1n|S\cap\{c_1,\ldots,c_n\}|\le k/n$, so $V_k(T)=k/n<1$.
\end{proof}

The construction is not artificial: goals whose statements share vocabulary with the same contracts, but whose proofs use different methods, have exactly this form.

\subsection{Candidate-pool ceiling}

\noindent\textbf{Corollary~\ref{cor:ceiling} (restated).}
\emph{Consider a two-stage selector that takes the top $m$ contracts $P_m(g)$ under a similarity-only selector and then orders them by any procedure, placing the remaining contracts after them. For every $k\le m$ and every goal $g$,
$\text{hit@}k\le\mathbf 1[P_m(g)\cap A(g)\neq\varnothing]$ and $\text{recall@}k\le|P_m(g)\cap A(g)|/|A(g)|$.}

\begin{proof}
By construction, the first $m$ positions of the final list are a permutation of $P_m(g)$, and the remaining positions hold $\mathcal C\setminus P_m(g)$. For $k\le m$, therefore, $\pi_k(g)\subseteq P_m(g)$, and hence $\pi_k(g)\cap A(g)\subseteq P_m(g)\cap A(g)$. The two inequalities follow by taking the indicator of nonemptiness and the cardinality of both sides. The right-hand sides are the first-stage selector's own hit@$m$ and recall@$m$ on $g$, since $P_m(g)$ is its top-$m$ list, and averaging over problems preserves both inequalities. The same containment bounds the reciprocal rank: if $P_m(g)\cap A(g)=\varnothing$, the first annotated contract appears at position $m+1$ or later, so $1/\operatorname{rank}(g)\le1/(m+1)$.
\end{proof}

\subsection{Top-$k$ ranking regret}

\noindent\textbf{Proposition~\ref{prop:regret} (restated).}
\emph{Let $S_k^*$ contain the $k$ contracts with the largest $w_c(x)$, let $\hat S_k$ contain the $k$ contracts with the largest scores $\hat s_c$, and suppose there is a strictly increasing map $f$ such that $|f(\hat s_c)-w_c(x)|\le\epsilon$ for every $c\in S_k^*\cup\hat S_k$. Then $R_x(S_k^*)-R_x(\hat S_k)\le2k\epsilon$.}

\begin{proof}
Write $w_c$ for $w_c(x)$ and set $\hat w_c=f(\hat s_c)$ for every $c\in\mathcal C$. Because $f$ is strictly increasing, $\hat s_c>\hat s_{c'}$ if and only if $\hat w_c>\hat w_{c'}$, so a set of $k$ contracts with the largest $\hat s_c$ is also a set of $k$ contracts with the largest $\hat w_c$. Consequently $\sum_{c\in\hat S_k}\hat w_c\ge\sum_{c\in S}\hat w_c$ for every $S$ with $|S|=k$, in particular for $S=S_k^*$. Now decompose
\[
R_x(S_k^*)-R_x(\hat S_k)
=\Bigl[\sum_{c\in S_k^*}\hat w_c-\sum_{c\in\hat S_k}\hat w_c\Bigr]
+\sum_{c\in S_k^*}(w_c-\hat w_c)
-\sum_{c\in\hat S_k}(w_c-\hat w_c).
\]
The bracket is at most $0$. Every $c$ in the two remaining sums lies in $S_k^*\cup\hat S_k$, where $|w_c-\hat w_c|\le\epsilon$ by assumption, so the second sum is at most $k\epsilon$ and the third is at least $-k\epsilon$. The difference is therefore at most $2k\epsilon$. Contracts in $S_k^*\cap\hat S_k$ cancel between the last two sums, so the bound sharpens to $2\,|S_k^*\setminus\hat S_k|\,\epsilon$.
\end{proof}

Three remarks. First, the map $f$ only needs to exist; it need not be known, so raw scores such as integers on a fixed scale are admissible as long as some monotone rescaling of them is close to $w_c(x)$ on the leading contracts. Second, the assumption constrains at most $2k$ contracts; estimates elsewhere in the library may be arbitrarily poor. Third, the sharpened bound says that the regret is proportional to the number of contracts on which the two shortlists disagree, which is the reading used in Section~\ref{sec:selfad}: only errors that move a contract into or out of the top $k$ matter.
            
\section{Contract library}

\subsection{A Method Contract}
Figure~\ref{fig:contract-example} shows an abridged contract from our library; auxiliary fields, such as variant rules and closing tactics, are omitted.
\begin{figure}[!ht]
\small
\begin{tabular}{@{}p{0.31\textwidth}p{0.66\textwidth}@{}}
\toprule
\multicolumn{2}{@{}l}{\textbf{Method Contract} \texttt{ANA.MVT}: Mean value theorem / Rolle \hfill domain: analysis} \\
\midrule
\multicolumn{2}{@{}l}{\textit{Selection side} (read by the selector)} \\[2pt]
Method & Obtain an interior point where the derivative equals the secant slope (Lagrange), or vanishes when the endpoint values agree (Rolle). \\
Typical target & A claim about $f'$ at an unspecified point, or a bound $|f(b)-f(a)|\le M(b-a)$ reached through $f'$. \\
Applies when & $f$ is real-valued, continuous on $[a,b]$, and differentiable on $(a,b)$, with $a<b$; for Rolle, also $f(a)=f(b)$. \\
Not for & Complex- or vector-valued $f$; sequences or integer-indexed data; goals that need a specific $c$. \\
Manifest & Scaffolds \texttt{mvt\_slope}, \texttt{mvt\_rolle}; anchors \texttt{exists\_deriv\_eq\_slope}, \texttt{exists\_deriv\_eq\_zero}, and variants; one worked instance. \\
\midrule
\multicolumn{2}{@{}l}{\textit{Execution side} (supplied to the prover)} \\[2pt]
Parameters $x_c$ & \texttt{f : ℝ → ℝ}, \texttt{a b : ℝ} \\
Required conditions $\operatorname{Pre}_c$ & \texttt{a < b}; for Rolle, \texttt{f a = f b}. \\
Residual obligations $\operatorname{Obl}_c$ & \texttt{ContinuousOn f (Set.Icc a b)}; for Lagrange, \texttt{DifferentiableOn ℝ f (Set.Ioo a b)}. \\
Target $G_c$ & \texttt{∃ c ∈ Set.Ioo a b, deriv f c = (f b - f a) / (b - a)} (Lagrange); \texttt{deriv f c = 0} (Rolle). \\
Scaffolds &
\begin{lstlisting}[style=lean]
theorem mvt_slope (f : ℝ → ℝ) (a b : ℝ) (hab : a < b)
    (hfc : ContinuousOn f (Set.Icc a b))
    (hfd : DifferentiableOn ℝ f (Set.Ioo a b)) :
    ∃ c ∈ Set.Ioo a b, deriv f c = (f b - f a) / (b - a) :=
  exists_deriv_eq_slope f hab hfc hfd

theorem mvt_rolle (f : ℝ → ℝ) (a b : ℝ) (hab : a < b)
    (hfI : f a = f b) (hfc : ContinuousOn f (Set.Icc a b)) :
    ∃ c ∈ Set.Ioo a b, deriv f c = 0 :=
  exists_deriv_eq_zero hab hfc hfI
\end{lstlisting} \\
Worked instance &
\begin{lstlisting}[style=lean]
-- instantiates mvt_slope at (f, 0, 1)
example (f : ℝ → ℝ)
    (hfc : ContinuousOn f (Set.Icc (0:ℝ) 1))
    (hfd : DifferentiableOn ℝ f (Set.Ioo (0:ℝ) 1))
    (h0 : f 0 = 0) (h1 : f 1 = 1) :
    ∃ c ∈ Set.Ioo (0:ℝ) 1, deriv f c = 1 := by
  obtain ⟨c, hc, hdc⟩ := mvt_slope f 0 1 (by norm_num) hfc hfd
  exact ⟨c, hc, by rw [hdc, h1, h0]; norm_num⟩
\end{lstlisting} \\
Verification & Lean v4.27.0, Mathlib \texttt{a3a10db}; axioms \texttt{propext}, \texttt{Classical.choice}, \texttt{Quot.sound}; no \texttt{sorry}. \\
\bottomrule
\end{tabular}
\caption{An example Method Contract. The selector reads the selection side; the prover receives the execution side. All Lean code compiles in the pinned environment (Appendix~\ref{app:env}).}
\label{fig:contract-example}
\end{figure}

\subsection{Lean environment}\label{app:env}      
\label{app:env}

All compile checks in this paper use the same environment: Lean v4.27.0 and Mathlib at commit \texttt{a3a10db0e9d66acbebf76c5e6a135066525ac900} (release v4.27.0). This covers the checks that certify each contract's execution side during construction (Section~\ref{sec:card-construction}) and every compilation in the proving experiments, so a scaffold that compiled when its contract was built remains valid when the prover uses it, as Proposition~\ref{prop:scaffold} requires. Every verifier instance loads precompiled Mathlib files from a single prebuilt environment image, and the runner exits at startup if the Mathlib revision recorded in the contract library differs from the one in the environment's Lake manifest, which prevents silent drift between the library and the verifier. Table~\ref{tab:lean-env} lists the pinned dependencies.

\begin{table}[h]
\centering
\small
\begin{tabular}{lll}
\toprule
Component & Commit & Version \\
\midrule
Lean toolchain & -- & v4.27.0 \\
Mathlib & \texttt{a3a10db0} & v4.27.0 \\
Aesop & \texttt{cb837cc2} & -- \\
Batteries & \texttt{b25b36a7} & -- \\
Qq & \texttt{bd58c9ef} & -- \\
ProofWidgets & \texttt{c04225ee} & v0.0.85 \\
ImportGraph & \texttt{8f497d55} & -- \\
LeanSearchClient & \texttt{5ce7f0a3} & -- \\
Plausible & \texttt{009dc1e6} & -- \\
Cli & \texttt{55c37290} & v4.27.0 \\
\bottomrule
\end{tabular}
\caption{Pinned Lean environment. Commits are abbreviated to eight characters; the full Mathlib commit is given in the text.}
\label{tab:lean-env}
\end{table}

\section{Experimental details}             
\subsection{Downstream}\label{app:downstream} 
We use GPT-5.6 Luna for proving, in a whole-proof repair loop with a budget of $B=10$ Lean compilations per problem. \emph{Base} is one attempt without retrieval; \emph{pass@20} draws 20 independent samples and is therefore given twice the budget; \emph{APOLLO}~\citep{ospanov2026apollo} alternates generation with Lean-guided repair; \emph{AxProverBase} is Ax-Prover~\citep{requena2026minimalagentautomatedtheorem} run without contracts under the same budget and serves as the direct control; \emph{+ Self-adv.\ contract (ours)} adds the top-ranked self-advertised contract to AxProverBase's initial prompt. All arms except pass@20 share the prover and the 10-compilation budget

\end{document}